\documentclass[10pt,journal]{IEEEtran}
\ifCLASSOPTIONcompsoc
\else
\fi

\ifCLASSINFOpdf
\else
\fi
\usepackage{graphicx}
\usepackage{amssymb}
\usepackage{multirow}
\usepackage{algorithm}
\usepackage{algorithmic}
\usepackage{amsmath,amsfonts}

\newcommand{\note}[1]{}
\newtheorem{definition}{Definition}[section]
\newtheorem{theorem}{Theorem}[section]
\newtheorem{lemma}{Lemma}[section]
\newtheorem{corollary}{Corollary}[section]

\begin{document}
\title{Blind Image Deblurring by Spectral Properties of Convolution Operators}
\author{Guangcan Liu, Shiyu Chang, and Yi Ma}
\maketitle
\begin{abstract}In this paper, we study the problem of recovering a sharp version of a given blurry image when the blur kernel is unknown. Previous methods often introduce an image-independent regularizer (such as Gaussian or sparse priors) on the desired blur kernel. We shall show that the blurry image itself encodes rich information about the blur kernel. Such information can be found through analyzing and comparing how the spectrum of an image as a convolution operator changes before and after blurring. Our analysis leads to an effective convex regularizer on the blur kernel which depends only on the given blurry image. We show that the minimizer of this regularizer guarantees to give good approximation to the blur kernel if the original image is sharp enough. By combining this powerful regularizer with conventional image deblurring techniques, we show how we could significantly improve the deblurring results through simulations and experiments on real images. In addition, our analysis and experiments help explaining a widely accepted doctrine; that is, the edges are good features for deblurring.
\end{abstract}
\IEEEdisplaynotcompsoctitleabstractindextext
\IEEEpeerreviewmaketitle
\section{Introduction}
Figure \ref{fig:deblur} shows a very common result of a very blurry image of a car taken by a moving camera --- or similar blurring effect can be observed in many surveillance photos where the camera is static but the car is moving. In many situations, we would like to recover the sharp version of the image so that details of the image (such as the numbers on the license plate) become recognizable to human eyes. It is in general impossible to correctly deblur the whole image if both the camera motion and the scene geometry are both entirely unknown. Nevertheless, if we only consider a small image region (e.g., the license plate area shown in Figure \ref{fig:deblur}), it is reasonable to assume that the blurry image, denoted as $B\in\mathbb{R}^{n_1\times{}n_2}$, is approximately generated by the convolution of a sharp image, denoted as $I_0$, and a generic blur kernel, denoted as $K_0\in\mathbb{R}^{m_1\times{}m_2}$:
\begin{eqnarray*}
B\approx{}I_0\otimes{}K_0,&\textrm{s.t.}& K_0\in\mathcal{S},
\end{eqnarray*}
where $\otimes$ denotes the discrete 2D convolution operator, $S$ denotes the \emph{simplex} of all possible blur kernels (i.e., nonnegative and sums to one), $\{n_1,n_2\}$ are the image sizes, and $\{m_1,m_2\}$ are the sizes of the blur kernel. So, the blind image deblurring problem can be mathematically formulated as the problem of \emph{blind deconvolution} \cite{Lucy:1972:JOCA,Ayers:1988,Kundur:1996:SPM}, which is to recover the sharp image $I_0$ when the blur kernel $K_0$ is unknown\footnote{In contrast, \emph{non-blind deconvolution} \cite{Krishnan:2009:NIPS} is the restoration of $I_0$ when $K_0$ is given.}.
\begin{figure}
\begin{center}
\includegraphics[width=0.48\textwidth]{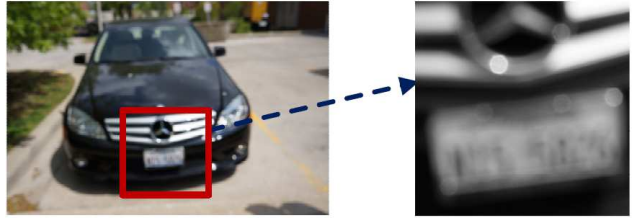}
\caption{\textbf{A scenario of image deblurring.} We take a picture of size $1624\times2448$ by using a NEX-5N camera. The picture is blurry due to an uncontrolled camera motion. We want to process the picture such that human eyes are able to recognize its details (e.g., the plate number).}\label{fig:deblur}
\end{center}
\end{figure}

The blind decovolution problem has been being investigated for several decades widely in optical society, image processing, computer graphics, and computer vision~\cite{Lucy:1972:JOCA,haichao:2011:icme,Cai:2009:CVPR,Neel:2008:CVPR,Levin:2011:TPAMI,Likas:2004:avariational,Wiener:1964:EIS}. Yet, this problem is severely ill-conditioned and still far from being well solved in the more general cases even with the notable progresses made recently (e.g., \cite{Cho:2011:CVPR,Jia:2007:Transprency,Levin:2006:NIPS,Li:2010:ECCV}). A straightforward approach for blind deconvolution is to jointly seek the sharp image $I_0$ and the blur kernel $K_0$ by minimizing
\begin{eqnarray*}
\min_{I,K}\|B-I\otimes{}K\|_F^2,&\textrm{s.t.}& K\in\mathcal{S},
\end{eqnarray*}
where $\|\cdot\|_F$ denotes the Frobenius norm of a matrix. However, this problem is highly ill-conditioned, as it can be perfectly minimized by infinite number of pairs $(I,K)$. For example, the \emph{no-blur explanation} is a perfect solution: $K=\delta$ (a delta function) and $I=B$. Indeed, even if the blur kernel $K_0$ is given, the non-blind deconvolution problem can still be ill-conditioned \cite{Krishnan:2009:NIPS,Tony:1998:TIP}. So, in general, it is necessary to regularize the desired solution for the image variable $I$~\cite{Fergus:2006:RCS}:
\begin{eqnarray}\label{eq:bldconv:regI}
\min_{I,K}\|B-I\otimes{}K\|_F^2+\lambda{}f(I),&\textrm{s.t.}& K\in\mathcal{S},
\end{eqnarray}
where $\lambda\geq0$ is a parameter. The regularizer $f(I)$ is usually chosen as the \emph{total variation} \cite{Jordan:1981:TV} or its variations \cite{Krishnan:2009:NIPS,Tony:1998:TIP,Yu:2000:TIP}. However, as analyzed in \cite{Levin:2011:TPAMI}, such image gradients based regularizer generally favors a blurry solution over a sharp one. More precisely, whenever $f(I)$ is convex, it is easy to prove that (see Section~\ref{sec:proof} for the proof)
\begin{eqnarray}\label{eq:regI:convexprob}
f(I\otimes{}K)\leq{}f(I), \quad \forall{}I\in\mathbb{R}^{n_1\times{}n_2}, \;\forall{}K\in\mathcal{S};
\end{eqnarray}
that is, minimizing $f(I)$ will encourage blurring, and thus the no-blur explanation (i.e., $I=B$ and $K=\delta$) is always favored by minimizing \eqref{eq:bldconv:regI}. Hence, it is also critical and indeed necessary to regularize the kernel $K$, i.e., it is necessary to consider an extended version of of \eqref{eq:bldconv:regI}:
\begin{eqnarray}\label{eq:bldconv:regIK}
\min_{I,K}\|B-I\otimes{}K\|_F^2+\lambda{}f(I)+\alpha{}h(K),\textrm{ s.t. } K\in\mathcal{S},
\end{eqnarray}
where $\lambda\geq0$ and $\alpha\geq0$ are two parameters.

While it is well-recognized that the regularizer $h(K)$ is important and indeed indispensable for blind deconvolution, the existing proposals for $h(K)$, e.g., the Gaussian function $h(K)=\|K\|_F^2$~\cite{Neel:2008:CVPR,Cho:2009:FMD}, the sparse regularizer $h(K)=\|K\|_1$ \cite{Li:2010:ECCV,Shan:2008:HMD} and the Bayesian prior~\cite{Levin:2011:TPAMI}, can work well only on some simple cases where there is no serious blur, but they could not handle more difficult deblurring tasks where the images are severely blurred (e.g., Figure \ref{fig:deblur}). Even more, a specific prior actually may not fit real world blur kernel, which is usually the combination of a sparse curve-like kernel (which corresponds to the camera motion) and a dense Gaussian-like kernel (that models effects such as out of focus).

In this work, we derive a much more effective regularizer for the blur kernel that can significantly benefit the solution of the blind deconvolution problem. This regularizer is based on a simple but important observation about the spectral properties of an image as a convolution operator: For a given image (i.e., matrix), consider its convolution with any other matrix. The convolution defines a linear operator. {\em Then empirically, the spectrum (the set of eigenvalues) of this linear operator for a blurry image is significantly smaller than that for its sharp counter part.} In fact, this can be proven to be true to some extend. Based on this observation, we devise a convex regularizer that tends to be minimized at the true blur kernel, called $K_0$. Namely, given an observed image $B$ represented by a certain image feature $\mathcal{L}$, we deduce a convex function, denoted as $h^{\mathcal{L}(B)}(K)$:
\begin{eqnarray}\label{eq:hbL}
h^{\mathcal{L}(B)}(K): \mathbb{R}^{m_1\times{}m_2}\rightarrow{}\mathbb{R}.
\end{eqnarray}
Unlike most previous work (e.g., \cite{Neel:2008:CVPR,Li:2010:ECCV,Cho:2009:FMD,Shan:2008:HMD}) where the regularizer $h(K)$ is independent of the observed image $B$, our regularizer $h^{\mathcal{L}(B)}(K)$ explicitly depends on the given blurry image and encodes information about how the blurry image $B$ is related to the sharp image $I_0$. It is hence somewhat natural to anticipate that this regularizer $h^{\mathcal{L}(B)}(K)$ would depend on the sharp image too. But rather surprisingly, as we will show, under fairly broad conditions, we can come up with a very effective regularizer $h^{\mathcal{L}(B)}(K)$ that does not depend on any information about the sharp image $I_0$ at all, and the desired kernel $K_0$ can be approximately retrieved by minimizing $h^{\mathcal{L}(B)}(K)$ directly.

Equipped with such new regularizer $h^{\mathcal{L}(B)}(K)$, for the blind deconvolution problem, we jointly seek the sharp image $I_0$ and the blur kernel $K_0$ by solving the optimization problem \eqref{eq:bldconv:regIK}. Experimental performance of our algorithm is a bit surprising: Even when the observed image is blurred to the extent that human eyes cannot recognize its details (e.g., Figure \ref{fig:deblur}), it is still possible for our algorithm to restore a sharp version with recognizable details. In addition to empirical evaluations, theoretical results presented in this paper could also help understanding the blind deconvolution problem: It is known that there are infinite number of ways to decompose a given blurry image $B$ into the convolution of $I$ and $K$. Nevertheless, while $I$ (or its edge map) is assumed to be very sharp, we show that such decomposition is actually unique (or approximately so). Moreover, our analysis provides an explanation for why sharp edges are good features for deblurring. That is, the spectrum of the images in the edge domain is more sensitive to blur than in the raw pixel domain. In summary, the contributions of this paper include:
\begin{itemize}
\item[$\diamond$] We establish a generic kernel regularizer that can be effectual for various blurs such as motion blur and de-focus. Unlike previous approaches which only have mild effects in constraining the solution of blind deconvolution, our regularizer has a strong effect and can even directly retrieve the blur kernel without knowing the latent image.
\item[$\diamond$] Our studies are helpful for understanding the blind deconvolution problem and explaining the widely accepted doctrine that the sharp edges of images are good features for deblurring.
\end{itemize}

\section{Blind Deconvolution via Spectral Properties of Image Convolution}
In this section, we present the details for designing the regularizer $h^{\mathcal{L}(B)}(K)$. The final algorithm for blind deconvolution will be given at the end of this section.
\subsection{Spectrum of a Natural Image as a Convolution Operator}\label{sec:eigen}
\subsubsection{Preliminaries}The concept of \emph{convolution} is well-known. We briefly introduce it for the ease of reading.
Let $X$ and $Y$ are functions of two discrete variables (i.e., $X$ and $Y$ are matrices), then the formula for the 2D convolution of $X$ and $Y$ is
\begin{eqnarray*}
(X\otimes{}Y)(i,j) = \sum_{u,v}X(i-u,j-v)Y(u,v),
\end{eqnarray*}
where $(\cdot)(i,j)$ denotes the $(i,j)$th entry of a matrix. The convolution operator is linear and can be converted into matrix multiplication. Let $\nu(\cdot)$ be the vectorization of a matrix. Then it can be calculated that
\begin{eqnarray}\label{eq:convmatrix}
\nu(X\otimes{}Y) = \mathcal{A}_{k_1,k_2}(X)\nu(Y),\quad \forall{}Y\in\mathbb{R}^{k_1\times{}k_2},
\end{eqnarray}
where $\mathcal{A}_{k_1,k_2}(\cdot)$ is the \emph{Toeplitz matrix} \cite{Filip:2005:TIP} of a matrix, and $k_1,k_2$ are parameters. For an $\ell_1$-by-$\ell_2$ matrix $X$, its Toeplitz matrix, denoted as $\mathcal{A}_{k_1,k_2}(X)$, is of size $(l_1+k_1-1)(l_2+k_2-1)$-by-$k_1k_2$.
\subsubsection{Convolution Eigenvalues and Eigenvectors of Images}
\begin{figure*}
\begin{center}
\includegraphics[width=1\textwidth]{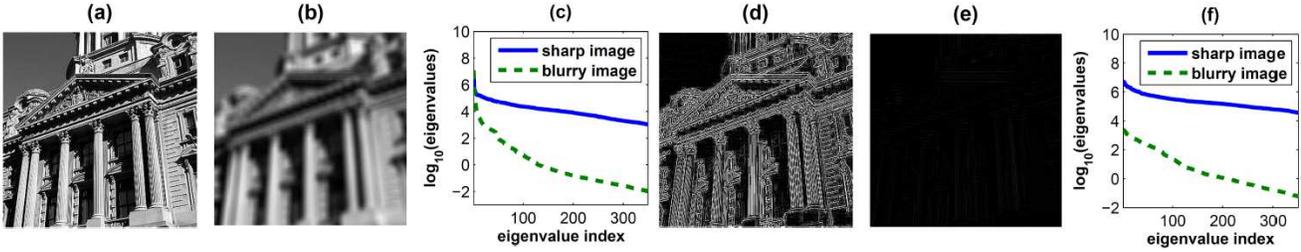}
\caption{\textbf{Example: blurring can significantly decreases the convolution eigenvalues of a sharp image}. (a) A sharp image. (b) The blurry image created by convoluting the sharp image with a Gaussian kernel. (c) Plots of all ($s_1=s_2=18$) convolution eigenvalues of the sharp and blurry images, using $\mathcal{L}=\delta$. (d) The edges of the sharp image, using LoG. (d) The edges of the blurry image (note that the black area contains negative values). (e) Plots of convolution eigenvalues of the sharp and blurry images, choosing $\mathcal{L} = $ LoG. } \label{fig:blureffects}
\end{center}
\end{figure*}
In this work, an image $I$ is treated as a matrix associated with a certain feature filter $\mathcal{L}$:
\begin{eqnarray}\label{eq:LI}
\mathcal{L}(I) = \mathcal{L}\otimes{}I.
\end{eqnarray}
Typically choices for the feature filter $\mathcal{L}$ include $\mathcal{L}=\delta$ (i.e., using original pixel values as features) or $\mathcal{L} = $ ``Laplacian of Gaussian (LoG)'' (i.e., using edge features).

To formally characterize the procedure of transforming sharp images into blurry images, we define the so-called \emph{convolution eigenvalues} and \emph{eigenvectors} as follows:
\begin{definition}\label{def:eigen}For an image $I\in\mathbb{R}^{l_1\times{}l_2}$ represented by a feature filter $\mathcal{L}$, its first convolution eigenvalue, denoted as $\sigma_1^{\mathcal{L}}(I)$ or $\sigma_{max}^{\mathcal{L}}(I)$, is defined by
\begin{eqnarray*}
\sigma_1^{\mathcal{L}}(I)=\max_{X\in\mathbb{R}^{s_1\times{}s_2}}\|\mathcal{L}(I)\otimes{}X\|_F,\textrm{ s.t. } \|X\|_F=1,
\end{eqnarray*}
where $s_1,s_2$ are called ``sampling sizes'' in this work. The maximizer to above problem is called the first convolution eigenvector, denoted as $\kappa_1^{\mathcal{L}}(I)$.

Similarly, the $i$th ($i=2,\ldots,s_1s_2$) convolution eigenvalue $\sigma_i^{\mathcal{L}}(I)$ is defined by
\begin{eqnarray*}\sigma_i^{\mathcal{L}}(I)&=&\max_{X\in\mathbb{R}^{s_1\times{}s_2}}\|\mathcal{L}(I)\otimes{}X\|_F, \\
 &\textrm{s.t.}& \|X\|_F=1, \langle{}X,\kappa_j^{\mathcal{L}}(I)\rangle=0,\forall{}j<i,
\end{eqnarray*}
where $\langle\cdot,\cdot\rangle$ denotes the inner production between two matrices. The maximizer to above problem is the $i$th convolution eigenvector, denoted as $\kappa_i^{\mathcal{L}}(I)$.
\end{definition}

In the following, we summarize some of the properties of the convolution eigenvalues and convolution eigenvectors which will be used later:
\begin{itemize}
\item[-] From \eqref{eq:convmatrix}, it can be seen that the convolution eigenvectors/values are exactly the right singular vectors/values of the Toeplitz matrix. So, for an image $I$ with the associated Toeplitz matrix $\mathcal{A}_{s_1,s_2}(\mathcal{L}(I))$, its all $s_1s_2$ convolution eigenvalues (and eigenvectors) can be found by computing the Singular Value Decomposition (SVD) of $(\mathcal{A}_{s_1,s_2}(\mathcal{L}(I)))^T\mathcal{A}_{s_1,s_2}(\mathcal{L}(I))$.
\item[-] Let $\sigma_{min}^{\mathcal{L}}(\cdot)$ denote the smallest (i.e., last) convolution eigenvalue of an image. Then
\begin{eqnarray}\label{eq:eigenmin}
\|\mathcal{L}(I)\otimes{}X\|_F\geq\sigma_{min}^{\mathcal{L}}(I)>0,\forall{}\|X\|_F=1,
\end{eqnarray}
which implies that the Toeplitz matrix is always nonsingular except the extreme case of $I=0$.
\item[-] If $B=I\otimes{}K$ and $K\in\mathcal{S}$, then (detailed proofs are in Section~\ref{sec:proof})
\begin{eqnarray}\label{eq:eigen:ineqs}
\sigma_{i}^{\mathcal{L}}(B)\leq\sigma_{i}^{\mathcal{L}}(I), \forall{} i=1,\ldots,s_1s_2;
\end{eqnarray}
that is, the blurring effects generally reduce the convolution eigenvalues of an image. When the original image is sharp, the reduction amount can be very significant, as exemplified in Figure \ref{fig:blureffects}. In particular, when the edge features are used, Figure \ref{fig:blureffects}(f) shows that it is even possible to have $\sigma_{min}^{\mathcal{L}}(I)\gg\sigma_{max}^{\mathcal{L}}(B)$.
\end{itemize}

The properties in inequality \eqref{eq:eigen:ineqs} suggest that a suitable way to formally define the concept of \emph{sharp image}, which appears frequently in the articles related to deconvolution, would be the following: An image $I$ is called \emph{$\tau$-sharp} if and only if $\sigma_{min}^{\mathcal{L}}(I) \geq\tau$, where $\tau>0$ is a parameter. That is, in general, the convolution eigenvalues of sharp images are large, while those of blurry images are relatively smaller.
\subsection{A Convex Blur Kernel Regularizer}
In this subsection, we derive a convex regularizer, denoted as $h^{\mathcal{L}(B)}(K)$, which tends to have the minimal value at the desired blur kernel $K_0$.
\subsubsection{Derivation}
For ease of exploration, we begin with the simple case that the blurry image $B$ is exactly generated by the convolution of $I_0$ and $K_0$, i.e., there is no noise: $B = I_0 \otimes{}K_0$. As convolution operators are associative and commutative, the effect of feature extraction is
\begin{eqnarray}\label{eq:conv:edge}
\mathcal{L}(B) = \mathcal{L}(I_0)\otimes{}K_0,
\end{eqnarray}
where $\mathcal{L}$ is a predefined feature filter. For the rest of this paper, we consistently choose $\mathcal{L} = $ LoG, i.e., we use edge features by default (but our analysis applies to any convolution filters).

By Definition \ref{def:eigen},
\begin{eqnarray*}
\|\mathcal{L}(B)\otimes{}\kappa_i^{\mathcal{L}}(B)\|_F=\sigma_i^{\mathcal{L}}(B), \; \forall{1\leq{}i\leq{}s_1s_2},
\end{eqnarray*}
where $\kappa_i^{\mathcal{L}}(B)\in\mathbb{R}^{s_1\times{}s_2}$, and the sampling sizes $\{s_1,s_2\}$ are taken as parameters. By \eqref{eq:conv:edge}, we have
\begin{eqnarray*}
\|(\mathcal{L}(I_0)\otimes{}K_0)\otimes{}\kappa_i^{\mathcal{L}}(B)\|_F=\sigma_i^{\mathcal{L}}(B), \forall{1\leq{}i\leq{}s_1s_2}.
\end{eqnarray*}
Since the convolution operator is linear, we further have
\begin{eqnarray*}
\Big\|\mathcal{L}(I_0)\otimes{}\frac{K_0\otimes\kappa_i^{\mathcal{L}}(B)}{\|K_0\otimes\kappa_i^{\mathcal{L}}(B)\|_F}\Big\|_F=\frac{\sigma_i^{\mathcal{L}}(B)}{\|K_0\otimes\kappa_i^{\mathcal{L}}(B)\|_F}.
\end{eqnarray*}
Notice that $\big\|\frac{K_0\otimes\kappa_i^{\mathcal{L}}(B)}{\|K_0\otimes\kappa_i^{\mathcal{L}}(B)\|_F}\big\|_F=1$. By \eqref{eq:eigenmin}, $\|\mathcal{L}(I_0)\otimes{}X\|_F\geq\sigma_{min}^{\mathcal{L}}(I_0),\forall{}\|X\|_F=1$. Thus we have the following necessary conditions that constrain possible $K_0$:
\begin{eqnarray}\label{eq:temp1}
\|K_0\otimes\kappa_i^{\mathcal{L}}(B)\|_F\leq\frac{\sigma_i^{\mathcal{L}}(B)}{\sigma_{min}^{\mathcal{L}}(I_0)}, \; \forall{1\leq{}i\leq{}s_1s_2}.
\end{eqnarray}

Define a function on the kernel $K$ as  $$h^{\mathcal{L}(B)}(K) \doteq \sum_{i=1}^{s_1s_2}\frac{\|K\otimes\kappa_i^{\mathcal{L}}(B)\|_F^2}{(\sigma_i^{\mathcal{L}}(B))^2}.$$
Then \eqref{eq:temp1} implies that
\begin{eqnarray}\label{eq:prior:k0}
h^{\mathcal{L}(B)}(K_0)\leq{}\frac{s_1s_2}{(\sigma_{min}^{\mathcal{L}}(I_0))^2}.
\end{eqnarray}
Note that $\sigma_i^{\mathcal{L}}(B)$ could be significantly smaller than $\sigma_{min}^{\mathcal{L}}(I_0)$ (see Figure \ref{fig:blureffects}), and a randomly chosen blur kernel may not satisfy \eqref{eq:temp1} or subsequently \eqref{eq:prior:k0}. Hence, the desired blur kernel $K_0$ should have relatively smaller value for $h^{\mathcal{L}(B)}(K)$. Or we could try to obtain an approximate estimate of the desired blur kernel $K_0$ by minimizing:
\begin{eqnarray}\label{eq:quadprog}
\hat{K}_0 = \arg\min_{K} h^{\mathcal{L}(B)}(K), &\textrm{s.t.}&K\in\mathcal{S},
\end{eqnarray}
where $S$ denotes the $(m_1m_2-1)$-dimensional simplex, and $m_1,m_2$ are the sizes of the blur kernel.

It is easy to see that $h^{\mathcal{L}(B)}(K)$ is a quadratical (hence convex) function, namely $h^{\mathcal{L}(B)}(K)=(\nu(K))^TH\nu(K)$ with the Hessian matrix $H$ given by
\begin{eqnarray}\label{eq:h}
H =\sum_{i=1}^{s_1s_2}\frac{(\mathcal{A}_{m_1,m_2}(\kappa_i^{\mathcal{L}}(B)))^T\mathcal{A}_{m_1,m_2}(\kappa_i^{\mathcal{L}}(B))}{(\sigma_i^{\mathcal{L}}(B))^2},
\end{eqnarray}
where $\mathcal{A}_{m_1,m_2}(\kappa_i^{\mathcal{L}}(B))$ is the Toeplitz matrix of the $i$th convolution eigenvector of $B$.
\subsubsection{Analysis}\label{sec:analysis}
We now study how effective the regularizer $h^{\mathcal{L}(B)}(K)$ is by evaluating how close its minimizer $\hat{K}_0$ is to the true kernel $K_0$ (proofs to all the theories can be found in Section~\ref{sec:proof}).
\begin{theorem}[Noiseless]\label{the:recovery:noiseless}
Suppose $B=I_0\otimes{}K_0, K_0\in\mathcal{S}$, and $I_0\neq0$. For the kernel $\hat{K}_0$ estimated by \eqref{eq:quadprog}, we have that
\begin{eqnarray*}
\|K_0-\hat{K}_0\|_F\leq\sqrt{2}\frac{\sigma_{max}^{\mathcal{L}}(B)}{\sigma_{min}^{\mathcal{L}}(I_0)}.
\end{eqnarray*}
\end{theorem}
\begin{figure}
\begin{center}
\includegraphics[width=0.48\textwidth]{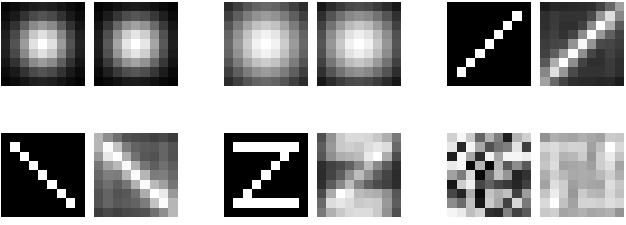}
\caption{\textbf{Demonstrating the effectiveness of the regularizer $h^{\mathcal{L}(B)}(K)$, using six synthetical examples.} Left: The blur kernel with size $9\times9$. Right: The kernel estimated by solving \eqref{eq:quadprog}.}\label{fig:res:ker}
\end{center}
\end{figure}
The above theorem illustrates that the edges are better than the raw pixels as a feature for the recovery of blur kernel, because the edge features can achieve a smaller bound for the estimate error (see Figure \ref{fig:blureffects}). This to some extent corroborates the previous observations (e.g., \cite{Cho:2011:CVPR,Li:2010:ECCV,Cho:2009:FMD,Fergus:2006:RCS,Shan:2008:HMD}) that the edge feature is a good choice for image deblurring. Note that in general $\sigma_{min}^{\mathcal{L}}(I_0)$ can be significantly larger than $\sigma_{max}^{\mathcal{L}}(B)$. For the example shown in Figure \ref{fig:blureffects}(f),  we have computed the ratio $\sigma_{max}^{\mathcal{L}}(B)/\sigma_{min}^{\mathcal{L}}(I_0)=0.15$, then the above theorem suggests that the estimation error is upper bounded by 0.21. Figure \ref{fig:res:ker} shows some more simulated examples of blur kernels directly estimated from a blurred Lena image, compared with the true kernels. Notice, that these kernels are obtained without any knowledge about the original image $I_0$ at all! This puts a strongly correct prior on the desired kernel entirely based on the given blurry image. These simulated results clearly demonstrate the effectiveness of the proposed regularizer $h^{\mathcal{L}(B)}(K)$.

Without any restrictions, it is known that there are infinite number of ways to decompose a blurry image $B$ into the convolution of an image $I$ and a blur kernel $K$. Interestingly, when the image $I$ (or its edge image) is assumed to be very sharp (i.e., the smallest convolution eigenvalue is large), Theorem \ref{the:recovery:noiseless} suggests that the decomposition for given image $B$ tends to be unique in a way that the allowable kernels should be very close to each other. The following corollary makes this statement more precise: We call an image $I$ is $\tau$-sharp if $ \sigma_{min}^{\mathcal{L}}(I)\geq\tau$ with $ \tau= 2\sqrt{2}\sigma_{max}^{\mathcal{L}}(B)/\varepsilon$ for some small number $\varepsilon$. Then we have the following:
\begin{corollary}\label{cor:unique} Denote all possible $\tau$-sharp decompositions of a blurry image $B$ as $\Omega_{B}^{\varepsilon}=\{(K,I)|I\otimes{}K = B,K\in\mathcal{S},\sigma_{min}^{\mathcal{L}}(I)\geq\sqrt{2}\sigma_{max}^{\mathcal{L}}(B)/\varepsilon\}$. For any two pairs $(K_0',I_0'),(K_0'',I_0'')\in\Omega_{B}^{\varepsilon}$, we have that
\begin{eqnarray*}
\|K_0'-K_0''\|_F\leq\varepsilon,
\end{eqnarray*}
where $\varepsilon>0$ is any small parameter.
\end{corollary}

Theorem\ref{the:recovery:noiseless} is based on the assumption that there is no noise, i.e., $B=I_0\otimes{}K_0$. Due to the fact the the blur in reality may not be uniform within an image, a more appropriate model is that $B=I_0\otimes{}K_0+N$, where $N$ denotes unknown noise (errors). In this case, we have the following theorem to bound the estimate error.
\begin{theorem}[Noisy]\label{the:recovery:noisy}
Suppose $B=I_0\otimes{}K_0+N, K_0\in\mathcal{S}$, $\|\mathcal{L}(N)\|_F\leq\epsilon$, and $I_0\neq0$. For the kernel $\hat{K}_0$ estimated by \eqref{eq:quadprog}, we have that
\begin{eqnarray*}
\|K_0-\hat{K}_0\|_F\leq\sqrt{2}\frac{\sigma_{max}^{\mathcal{L}}(B)+ccond(B)\sqrt{s_1s_2}\epsilon}{\sigma_{min}^{\mathcal{L}}(I_0)},
\end{eqnarray*}
where $ccond(B)=\sigma_{max}^{\mathcal{L}}(B)/\sigma_{min}^{\mathcal{L}}(B)$ is the ``convolution condition number'' of $B$.
\end{theorem}
\subsection{Reliable Deconvolution via Alternating Minimization}
While elegant and effectual, as can be seen from Figure~\ref{fig:res:ker}, the convex program \eqref{eq:quadprog} alone may not produce a satisfactory kernel which is good enough for restoring a sharp, natural image --- We have experimentally confirmed that the blur kernel from \eqref{eq:quadprog} often leads to very sharp, but unnatural images full of artifacts. For reliable deconvolution, we turn to use the proposed regularizer $h^{\mathcal{L}(B)}(K)$ as an additional term to constrain the deconvolution problem. Namely, we jointly seek the sharp image $I_0$ and the blur kernel $K_0$ by minimizing the following objective function:
\begin{align*}
\min_{I,K}\|B&-I\otimes{}K\|_F^2\\
&+\lambda\|\nabla{}I\|_1+\alpha{}h^{\mathcal{L}(B)}(K),\textrm{s.t. }K\in\mathcal{S},
\end{align*}
where is to choose $f(I)=\|\nabla{}I\|_1$ (i.e., total variation) and $h(K)=h^{\mathcal{L}(B)}(K)$ in \eqref{eq:bldconv:regIK}.

Although the above problem is nonconvex in nature and may heavily depend on the initial solution chosen in advance, it is now equipped with our strong kernel regularizer, $h^{\mathcal{L}(B)}(K)$, and thus the initialization step is no long crucial: Unlike the previous algorithms (e.g., \cite{Neel:2008:CVPR,Cho:2011:CVPR,Li:2010:ECCV,Cho:2009:FMD,Shan:2008:HMD}) which need to carefully chose the initial solution, we simply choose the observed blurry image $B$ to be the initial condition for $I$.
\begin{itemize}
\item While fixing the variable $I$, the blur kernel $K$ is updated by solving $\min_{K}\|B-I\otimes{}K\|_F^2+\alpha{}h^{\mathcal{L}(B)}(K)$, s.t. $K\in\mathcal{S}$, which is equal to the following quadratical programming:
\begin{align*}
&\min_{K}\|\nu(B)-\mathcal{A}_{m_1,m_2}(I)\nu(K)\|^2+\alpha{}(\nu(K))^TH\nu(K),\\
&\textrm{s.t.  }K\in\mathcal{S},
\end{align*}
where the Hessian matrix $H$ is computed by \eqref{eq:h}, $\|\cdot\|$ is the $\ell_2$-norm of a vector, and $\nu(\cdot)$ denotes the vectorization of a matrix.
\item While fixing the blur kernel $K$, the estimate of the image $I$ is updated by
\begin{eqnarray*}
\min_{I}\|B-I\otimes{}K\|_F^2+\lambda\|\nabla{}I\|_1,
\end{eqnarray*}
which can be solved by any of the many non-blind deconvolution algorithms developed in the literature (e.g.,~\cite{Krishnan:2009:NIPS,Yuan:2008:atg,Cho:2011:HON}). In this paper, we simply use the fast method introduced by Krishnan and Fergus~\cite{Krishnan:2009:NIPS}.
\end{itemize}
Furthermore, unlike previous blind deblurring methods that need to carefully control the number of iterations, we run the iterations until convergence. Usually, our algorithm needs about 100 iterations to converge.\\
\\\textbf{On Choosing the Parameters.} There are six parameters in total: The kernel sizes $\{m_1,m_2\}$ (usually $m_1=m_2$), the sampling sizes $\{s_1,s_2\}$, and the trade-off parameters $\alpha,\lambda$. Three of them, including $\{m_1,m_2\}$ and $\alpha$, need be set carefully. The kernel sizes $\{m_1,m_2\}$ are better to be a bit larger than the true blur size, which requires trying the algorithm several times to determine. It also needs some efforts to choose the parameter $\alpha$. When $\alpha$ is too small, our algorithm will always converge to the no-blur explanation (i.e., $I=B,K=\delta$); when $\alpha$ is set to be very large, the recovered image is very sharp, but unnatural. Usually, there exists such a threshold $\alpha^*$: The optimal solution is always $(I=B,K=\delta)$ while $\alpha<\alpha^*$; the solution is often satisfactory while $\alpha$ is slightly larger than $\alpha^*$.

After the kernel sizes being determined, we consistently set the sampling sizes as $s_1=1.5m_1$ and $s_2=1.5m_2$. The parameter $\lambda$ needs no extensive adjustments. In our experiments, this parameter is chosen from the range of 0.001 to 0.002.\\
\begin{figure*}
\begin{center}
\includegraphics[width=0.9\textwidth]{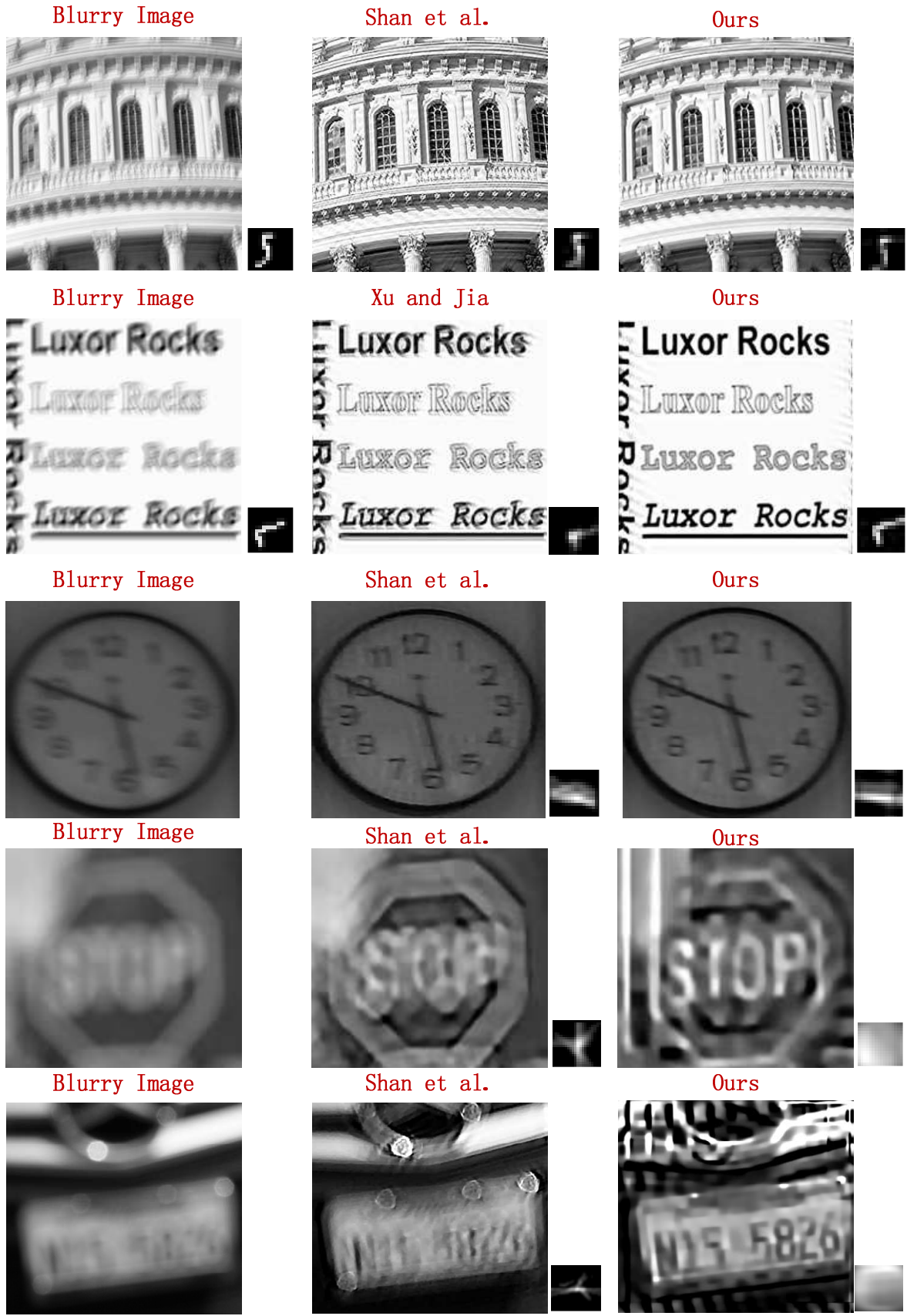}
\caption{\textbf{Comparison results, using two synthetical and three real examples.} Left: The observed blurry image. Middle: The best result among Fergus et al. \cite{Fergus:2006:RCS}, Shan et al. \cite{Shan:2008:HMD}, Cho and Lee \cite{Cho:2009:FMD}, Xu and Jia \cite{Li:2010:ECCV}, and Cho et al. \cite{Cho:2011:CVPR}. Right: Our result.  For these five examples, the parameter $\alpha$ is set as 5050, 1050, 70, 4000, and 27000, respectively; the kernel sizes are, $13\times13$, $13\times13$, $17\times17$, $17\times17$, and $31\times31$, respectively. The parameters of the baselines are also manually tuned to best.}\label{fig:res}
\end{center}
\end{figure*}
\\\textbf{Computational Cost.} Due to the developments of fast non-blind deconvolution (e.g., \cite{Krishnan:2009:NIPS}), the procedure for updating the image variable $I$ is already very fast. So we only analyze the costs for computing the Hessian matrix $H$ and updating the kernel variable $K$. For simplicity, we assume that $s_1=m_1,s_2=m_2$. Then the complexity of computing the Hessian matrix $H$ is $O(m_1^2m_2^2n_1n_2+m_1^3m_2^3)$, where $n_1,n_2$ are image sizes. In each iteration, the Toeplitz matrix $\mathcal{A}_{m_1,m_2}(I)$ need be updated. So, the complexity of the blind deconvolution procedure is $O(n_s(m_1m_2n_1n_2+m_1^3m_2^3))$, where $n_s$ is the number of iterations needed to converge. Overall, our current algorithm is not fast while deal with large kernels. Nevertheless, it is possible to speed up the algorithm by fast Fourier transforms. We leave this as future work.

\section{Experiments}
\subsection{Results}
We test with five examples (two synthetical, three real): The synthetical images are the convolution of $300\times300$ natural images and $13\times13$ synthetical blur kernels; the real images are subregions (about $350\times350$) selected from large pictures captured by a NEX-5N camera (please refer to Figure \ref{fig:deblur}). To show the advantages of the proposed method, we also test five state-of-the-art blind deconvolution algorithms, including \cite{Cho:2011:CVPR}, \cite{Cho:2009:FMD}, \cite{Fergus:2006:RCS}, \cite{Shan:2008:HMD}, and \cite{Li:2010:ECCV}.

Figure \ref{fig:res} shows the comparison results. To save space, we only show the best result of the previous algorithms for comparision. On the simple examples with easy blur kernels and tiny noise level (the first three examples in Figure \ref{fig:res}), our algorithm performs as well as the most effective baseline. While dealing with challenging cases (the last two examples in Figure \ref{fig:res}), where the blur kernels are complicated, it can be seen that our algorithm works distinctly better than the most competitive baseline. This is because that our regularizer $h^{\mathcal{L}(B)}(K)$ contains strong and accurate priors about the blur kernels, whereas the kernel priors adopted by previous algorithms are too weak to handle such difficult deblurring tasks. In particular, the forth and fifth examples illustrate that it is possible for our algorithm to successfully handle some extremely difficult cases, where the blurry images are very unclear such that human eyes are unable to recognize their contents.
\subsection{Diagnosis}
A basic assumption of our approach is that the edges of a sharp image are different before and after blurring. So our algorithm is less effective while the images contain fewer edges, as shown in Figure \ref{fig:diag:1}. Particularly, the third example of Figure \ref{fig:diag:1} shows that the ``true'' blur kernel cannot be accurately recovered if the original image has very few edges. Of course, it is not very critical to consider such images, as there is almost no difference between the sharp and blurry versions.

Note that the forth example of Figure \ref{fig:res} may have non-uniform blur kernels, since the deblurred result favors the characters ``O,P'' over ``S,T''. This is true, as verified in Figure \ref{fig:diag:2}: While the parameter $\alpha$ is varying, the shape of the estimated blur kernel changes as well. These results also reflex that our algorithm (with uniform kernel) is stable --- The algorithm does not crash even on the nonuniform cases.
\begin{figure}
\begin{center}
\includegraphics[width=0.47\textwidth]{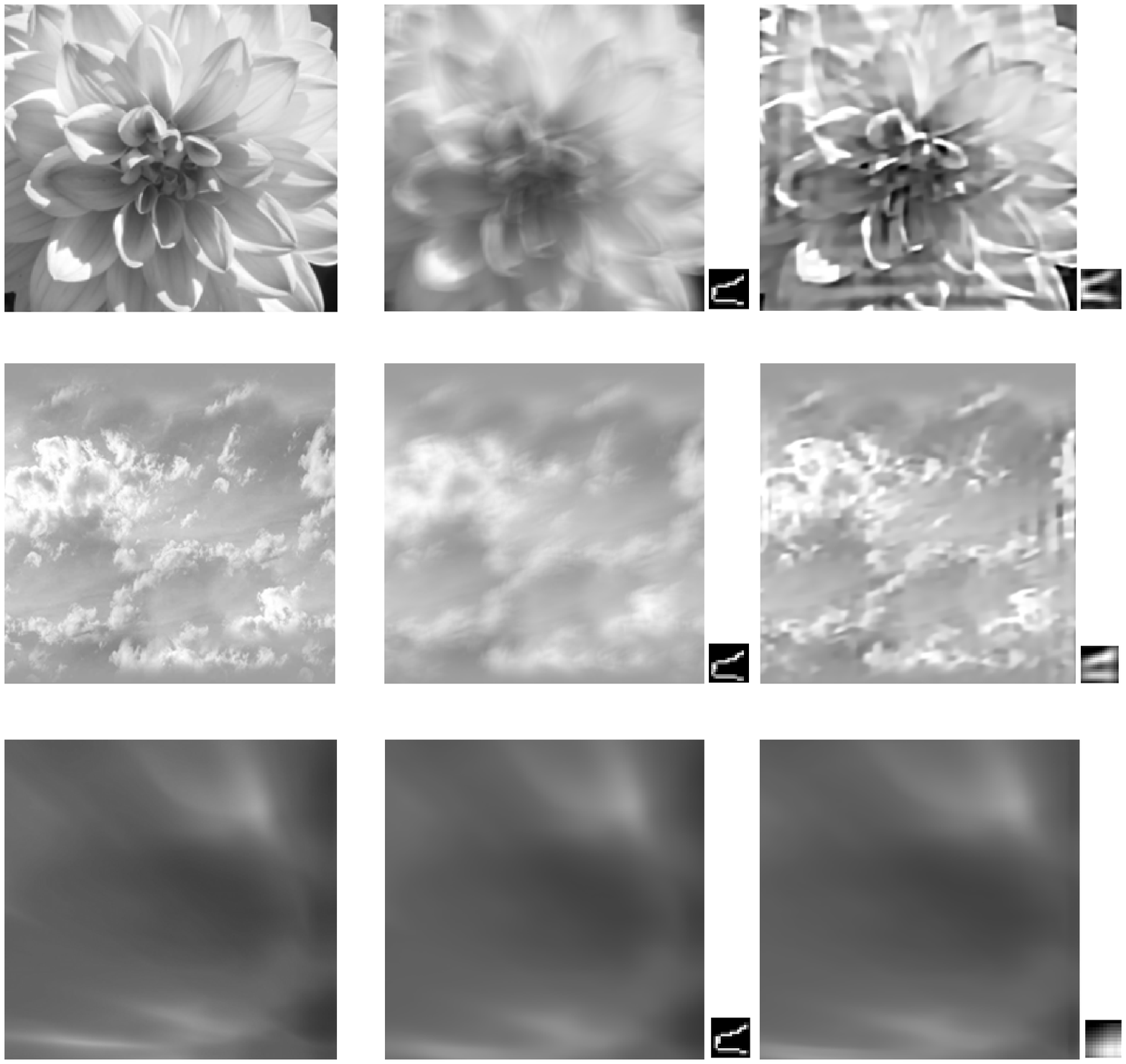}
\caption{\textbf{Some results obtained from the images with few edges.} Left: The original sharp image. Middle: The blurred version of the sharp image. Right: The deblurred version produced by our algorithm.}\label{fig:diag:1}
\end{center}
\end{figure}
\begin{figure}
\begin{center}
\includegraphics[width=0.48\textwidth]{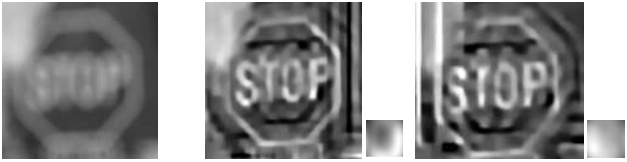}
\caption{\textbf{Non-uniform blur kernels.} Left: The observed blurry image. Middle: The deblurred version produced by our algorithm with $\alpha=1000$. Right: The deblurred version produced by our algorithm with $\alpha=4000$.}\label{fig:diag:2}
\end{center}
\end{figure}
\section{Conclusions and Discussions}
By studying how the spectrum of an image as a convolution operator changes before and after blurring, we have derived a convex regularizer on the blur kernel that depends on the given blurry image. For the blind deconvolution problem, we show that this convex regularizer is an effective prior that can deal with various types of realistic and challenging blurs (e.g., a combination of Gaussian and motion blurs). Both theoretical and experimental results have verified the validity and effectiveness of the proposed prior. Notice that our regularizer harnesses certain necessary conditions on the blur kernel, but not sufficient. So it is entirely possible there might exist other, potentially more effective image-dependent regularizers for the blur kernel.

Theorem~\ref{the:recovery:noiseless} and the empirical results in Figure \ref{fig:res:ker} even suggest that in theory, for certain class of sharp images, it might be possible to solve the blind deconvolution problem by avoiding the iterative minimization altogether. Namely, the problem might be solvable by using only two computational stable procedures: Firstly estimate the blur kernel without estimating the sharp image at all; and then recover the sharp image by performing non-blind deconvolution. However, we leave such investigations for future.
\section{Proofs}\label{sec:proof}
\subsection{Proof of the Inequality \eqref{eq:regI:convexprob}}
\begin{proof}Let $\mathcal{T}_{u,v}(\cdot)$ denote the transformation operator that shifts a 2D function $I(\cdot,\cdot)$ from $I(x,y)$ to $I(x-u,y-v)$. By the convexity of $f(\cdot)$,
\begin{eqnarray*}
f(I\otimes{}K)&=&f(\sum_{u,v}K(u,v)\mathcal{T}_{u,v}(I))\\
&\leq&\sum_{u,v}K(u,v)f(\mathcal{T}_{u,v}(I))\\
&=&f(I),
\end{eqnarray*}
where the last equality is due to the fact $f(\mathcal{T}_{u,v}(I))=f(I)$. This is naturally satisfied, since $\mathcal{T}_{u,v}(I)$ and $I$ refer to the same image.
\end{proof}
\subsection{Proof of the Inequality \eqref{eq:eigen:ineqs}}
\begin{proof}Using the ``min-max'' half of the Courant-Fisher theorem, we have that
\begin{eqnarray*}
\sigma_{i}^{\mathcal{L}}(B)&=&\min_{\substack{Y_j,\\j=1,\cdots,i-1.}}\max_{\substack{\|X\|_F=1, \\ \langle{}X,Y_j\rangle=0,\\j=1,\cdots,i-1.}}\|\mathcal{L}(B)\otimes{X}\|_F\\
&\leq&\min_{\substack{Y_j=\kappa_j^{\mathcal{L}}(I),\\j=1,\cdots,i-1.}}\max_{\substack{\|X\|_F=1, \\ \langle{}X,Y_j\rangle=0,\\j=1,\cdots,i-1.}}\|\mathcal{L}(B)\otimes{X}\|_F\\
&=&\max_{\substack{\|X\|_F=1, \\ \langle{}X,\kappa_j^{\mathcal{L}}(I)\rangle=0,\\j=1,\cdots,i-1.}}\|\mathcal{L}(B)\otimes{X}\|_F.
\end{eqnarray*}
By \eqref{eq:regI:convexprob}, we have $\|\mathcal{L}(B)\otimes{}X\|_F\leq\|\mathcal{L}(I)\otimes{}X\|_F$ and thus
\begin{eqnarray*}
\sigma_{i}^{\mathcal{L}}(B)&\leq&\max_{\substack{\|X\|_F=1, \\ \langle{}X,\kappa_j^{\mathcal{L}}(I)\rangle=0,\\j=1,\cdots,i-1.}}\|\mathcal{L}(B)\otimes{X}\|_F\\
&\leq&\max_{\substack{\|X\|_F=1, \\ \langle{}X,\kappa_j^{\mathcal{L}}(I)\rangle=0,\\j=1,\cdots,i-1.}}\|\mathcal{L}(I)\otimes{X}\|_F\\
&=&\sigma_i^{\mathcal{L}}(I),
\end{eqnarray*}
where the last equality is due to the definition of the convolution eigenvalues.
\end{proof}
\subsection{Proof of Theorem~\ref{the:recovery:noiseless}}
We first establish the following lemma:
\begin{lemma}\label{lemma:H}
If $B\neq0$, then the Hessian matrix $H$ defined by (13) is positive definite and obeys
\begin{eqnarray*}
h^{\mathcal{L}(B)}(X)\geq{}\frac{s_1s_2}{(\sigma_{max}^{\mathcal{L}}(B))^2}\|X\|_F^2,\forall{}X\in\mathbb{R}^{m_1\times{}m_2}.
\end{eqnarray*}
\end{lemma}
\begin{proof} It could be calculated that
\begin{align*}
&h^{\mathcal{L}(B)}(X) = \sum_{i=1}^{s_1s_2}\frac{\|X\otimes\kappa_i^{\mathcal{L}}(B)\|_F^2}{(\sigma_i^{\mathcal{L}}(B))^2}\\ &\geq\frac{1}{(\sigma_{max}^{\mathcal{L}}(B))^2}\sum_{i=1}^{s_1s_2}\|X\otimes\kappa_i^{\mathcal{L}}(B)\|_F^2\doteq\frac{(\nu(X))^TH_1\nu(X)}{(\sigma_{max}^{\mathcal{L}}(B))^2},
\end{align*}
where $H_1=\sum_{i=1}^{s_1s_2}(\mathcal{A}_{m_1,m_2}(\kappa_i^{\mathcal{L}}(B)))^T\mathcal{A}_{m_1,m_2}(\kappa_i^{\mathcal{L}}(B))$. The Toeplitz matrix $\mathcal{A}_{m_1,m_2}(\kappa_i^{\mathcal{L}}(B))$ is a linear operator of $\kappa_i^{\mathcal{L}}(B)$ and can be explicitly written as
\begin{eqnarray*}
\mathcal{A}_{m_1,m_2}(\kappa_i^{\mathcal{L}}(B))=[(\nu(\kappa_i^{\mathcal{L}}(B)))^T\Omega_1;(\nu(\kappa_i^{\mathcal{L}}(B)))^T\Omega_2;\cdots],
\end{eqnarray*}
where $\{\Omega_j\}_{j=1}^{(m_1+s_1-1)(m_2+s_2-1)}$ is a set of binary matrices and satisfies $\sum_{j}(\Omega_j)^T\Omega_j=s_1s_2\mathcal{I}$, and $\mathcal{I}$ is the $m_1\times{}m_2$ identity matrix. Hence,
\begin{align*}
H_1&=\sum_{i}\sum_{j}(\Omega_j)^T\nu(\kappa_i^{\mathcal{L}}(B))(\nu(\kappa_i^{\mathcal{L}}(B)))^T\Omega_j\\
&=\sum_{j}(\Omega_j)^T(\sum_{i}\nu(\kappa_i^{\mathcal{L}}(B))(\nu(\kappa_i^{\mathcal{L}}(B)))^T)\Omega_j\\
&=\sum_{j}(\Omega_j)^T\Omega_j,\\
&=s_1s_2\mathcal{I},
\end{align*}
which finishes the proof.
\end{proof}
\begin{proof}{\textbf{of Theorem~\ref{the:recovery:noiseless}}}
By \eqref{eq:prior:k0} and the convexity of the function $h^{\mathcal{L}(B)}(\cdot)$,
\begin{eqnarray*}
h^{\mathcal{L}(B)}(K_0-\hat{K}_0)&\leq&h^{\mathcal{L}(B)}(K_0)+h^{\mathcal{L}(B)}(\hat{K}_0)\\
&\leq&2h^{\mathcal{L}(B)}(K_0)\\
&\leq&\frac{2s_1s_2}{(\sigma_{min}^{\mathcal{L}}(I_0))^2}.
\end{eqnarray*}
By Lemma~\ref{lemma:H}, we also have
\begin{eqnarray*}
h^{\mathcal{L}(B)}(K_0-\hat{K}_0)\geq\frac{s_1s_2}{(\sigma_{max}^{\mathcal{L}}(B))^2}\|K_0-\hat{K}_0\|_F^2.
\end{eqnarray*}
Hence, $\|K_0-\hat{K}_0\|_F^2\leq2(\sigma_{max}^{\mathcal{L}}(B))^2/(\sigma_{min}^{\mathcal{L}}(I_0))^2.$
\end{proof}
\subsection{Proof of Corollary~\ref{cor:unique}}
\begin{proof} Since $\hat{K}_0$ is deterministic, we have
\begin{eqnarray*}
\|K_0'-K_0''\|_F&=&\|(K_0'-\hat{K}_0)+(\hat{K}_0-K_0'')\|_F\\
&\leq&\|K_0'-\hat{K}_0\|_F+\|\hat{K}_0-K_0''\|_F\\
&\leq&\sqrt{2}\frac{\sigma_{max}^{\mathcal{L}}(B)}{\sigma_{min}^{\mathcal{L}}(I_0')}+\sqrt{2}\frac{\sigma_{max}^{\mathcal{L}}(B)}{\sigma_{min}^{\mathcal{L}}(I_0'')}\\
&\leq&\frac{\varepsilon}{2}+\frac{\varepsilon}{2}=\varepsilon.
\end{eqnarray*}
\end{proof}
\subsection{Proof of Theorem~\ref{the:recovery:noisy}}
We need the following lemma to accomplish the proof.
\begin{lemma} For any two matrices $X$ and $Y$, we have
\begin{eqnarray*}
\|X\otimes{}Y\|_F&\leq&\|X\|_F\|Y\|_1,
\end{eqnarray*}
where $\|\cdot\|$ is the $\ell_1$-norm of a matrix.
\end{lemma}
\begin{proof}
Decompose $Y$ into the sum of its negative part (denoted as $Y^{-}$) and nonnegative part (denoted as $Y^{+}$). By \eqref{eq:regI:convexprob} and the convexity of matrix norms,
\begin{eqnarray*}
\|X\otimes{}Y\|_F&=&\|X\otimes{}Y^{+}+X\otimes{}Y^{-}\|_F\\
&\leq&\|X\otimes{}Y^{+}\|_F+\|X\otimes{}Y^{-}\|_F\\
&\leq&\|X\|_F\|Y^{+}\|_1+\|X\|_F\|Y^{-}\|_1\\
&=&\|X\|_F\|Y\|_1.
\end{eqnarray*}
\end{proof}

\begin{proof}{\textbf{Of Theorem~\ref{the:recovery:noisy}}}
In noisy case, the inequality \eqref{eq:temp1} needs be changed to
\begin{eqnarray*}
\|K_0\otimes\kappa_i^{\mathcal{L}}(B)\|_F&\leq&\frac{\sigma_i^{\mathcal{L}}(B)+\|\mathcal{L}(N)\otimes{}\kappa_i^{\mathcal{L}}(B)\|_F}{\sigma_{min}^{\mathcal{L}}(I_0)}\\
&\leq&\frac{\sigma_i^{\mathcal{L}}(B)+\epsilon\|\kappa_i^{\mathcal{L}}(B)\|_1}{\sigma_{min}^{\mathcal{L}}(I_0)}\\
&\leq&\frac{\sigma_i^{\mathcal{L}}(B)+\sqrt{s_1s_2}\epsilon}{\sigma_{min}^{\mathcal{L}}(I_0)}.
\end{eqnarray*}
Thus the inequality \eqref{eq:prior:k0} becomes
\begin{eqnarray*}
h^{\mathcal{L}(B)}(K_0)&\leq&\sum_{i}\frac{(1+\frac{\sqrt{s_1s_2}\epsilon}{\sigma_i^{\mathcal{L}}(B)})^2}{(\sigma_{min}^{\mathcal{L}}(I_0))^2}\\
&\leq&\frac{s_1s_2(1+\frac{\sqrt{s_1s_2}\epsilon}{\sigma_{min}^{\mathcal{L}}(B)})^2}{(\sigma_{min}^{\mathcal{L}}(I_0))^2}.
\end{eqnarray*}
Then by following the proof procedure of Theorem~\ref{the:recovery:noiseless}, this proof can be finished.
\end{proof}
\bibliographystyle{IEEETran}

\bibliography{ref}
\end{document}